%% file: Final_version.tex
\documentclass[10pt]{article}
\usepackage{spconf}
\usepackage{verbatim}
\usepackage{hyperref} % create hyperlinks
\usepackage{graphicx, caption}
\usepackage{subcaption}
\usepackage{amssymb}
\usepackage{amsmath}
\usepackage{amsthm}
\usepackage{mathtools}
\usepackage{cite}
\usepackage{stfloats}
\usepackage{epstopdf}
\usepackage{psfrag}
\usepackage[mathscr]{euscript}
\usepackage{acronym}  % make an acronym
\usepackage{booktabs} 
\usepackage[linesnumbered,ruled,vlined]{algorithm2e}
\usepackage[table]{xcolor}
\usepackage{multirow}
\usepackage[fontsize=9pt]{fontsize}

\SetKwInput{KwInput}{Input}
\SetKwInput{KwOutput}{Output}

\usepackage{color}
\usepackage{dsfont}
\usepackage{bbm}

\input{SupportDocuments/acronym}

\input{SupportDocuments/defmetric}

\title{Analysis of  the Memorization and Generalization \\ Capabilities of AI Agents: Are Continual Learners Robust? \vspace{-0.0cm}}
\name{Minsu Kim and Walid Saad \vspace{-0.0cm}
	\thanks{This research was supported by a grant from the Amazon-Virginia Tech Initiative for Efficient and Robust Machine Learning.}
	\thanks{The source code is available on https://github.com/news-vt}
}

\address{Wireless@VT, Bradley Department of Electrical and Computer Engineering, \\
Virginia Tech, Arlington, VA, USA.
\vspace{-0.0cm}}
\begin{document}
%\ninept
%
\maketitle
\begin{abstract}
	\vspace{-0.0cm}
In continual learning (CL), an AI agent (e.g., autonomous vehicles or robotics) learns from non-stationary data streams under dynamic environments. For the practical deployment of such applications, it is important to guarantee robustness to unseen environments while maintaining past experiences. In this paper, a novel CL framework is proposed to achieve robust generalization to dynamic environments while retaining past knowledge. The considered CL agent uses a capacity-limited memory to save previously observed environmental information to mitigate forgetting issues. Then, data points are sampled from the memory to estimate the distribution of risks over environmental change so as to obtain predictors that are robust with unseen changes. The generalization and memorization performance of the proposed framework are theoretically analyzed. This analysis showcases the tradeoff between memorization and generalization with the memory size. Experiments show that the proposed algorithm outperforms memory-based CL baselines across all environments while significantly improving the generalization performance on unseen target environments.
\end{abstract}
\vspace{-0.0cm}
\begin{keywords}
Robustness, Generalization, Memorization, Continual Learning
\vspace{-0.0cm}
\end{keywords}
\section{Introduction}
\vspace{-0.0cm}
\label{sec:intro}
\Ac{CL} recently emerged as a new paradigm for designing \ac{AI} systems that are adaptive and self-improving over time \cite{lopez2017gradient}. In \ac{CL}, \ac{AI} agents continuously learn from non-stationary data streams and adapt to dynamic environments. As such, \ac{CL} can be applied to many real-time \ac{AI} applications such as autonomous vehicles or digital twins \cite{omar}. For these applications to be effectively deployed, it is important to guarantee robustness to unseen environments while retaining past knowledge. However, modern deep neural networks often forget previous experiences after learning new information and struggle when faced with changes in data distributions \cite{eastwood2022probable}. Although \ac{CL} can mitigate forgetting issues, ensuring both memorization and robust generalization to unseen environments is still a challenging problem.

%However, modern deep neural networks often forget previous experiences after learning new information and struggle when faced with changes in data distributions \cite{eastwood2022probable}. 

To handle forgetting issues in \ac{CL}, many practical approaches have been proposed using memory \cite{rolnick2019experience, buzzega2020dark, aljundi2019gradient, chaudhry2021using, sun2022information} and regularization \cite{schwarz2018progress, pourkeshavarzi2021looking}. For memory-based methods \cite{rolnick2019experience, buzzega2020dark, aljundi2019gradient, chaudhry2021using, sun2022information}, a memory is deployed to save past data samples and to replay them when learning new information. Regularization-based methods \cite{schwarz2018progress, pourkeshavarzi2021looking} use regularization terms during model updates to avoid overfitting the current environment. However, these approaches  \cite{rolnick2019experience, buzzega2020dark, aljundi2019gradient, chaudhry2021using, sun2022information, schwarz2018progress, pourkeshavarzi2021looking} did not consider or theoretically analyze generalization performance on unseen environments even though they targeted non-stationary data streams. 

Recently, a handful of works \cite{peng2023ideal, raghavan2021formalizing, lin2023theory, cai2021online, machireddy2022continual, guo2023out} studied the generalization of \ac{CL} agents. In \cite{peng2023ideal}, the authors theoretically analyzed the generalization bound of memory-based \ac{CL} agents. The work in \cite{raghavan2021formalizing} used game theory to investigate the tradeoff between generalization and memorization. In \cite{lin2023theory}, the authors analyzed generalization and memorization under overparameterized linear models. However, the works in \cite{peng2023ideal, raghavan2021formalizing, lin2023theory} require that an \ac{AI} agent knows when and how task/environment identities, e.g., labels, will change. In practice, such information is usually unavailable and unpredictable. Hence, we focus on more general \ac{CL} settings \cite{buzzega2020dark} without such assumption. Meanwhile, the works in \cite{cai2021online, machireddy2022continual, guo2023out} empirically analyzed generalization under general \ac{CL} settings. However, they did not provide a theoretical analysis of the tradeoff between generalization and memorization performance.

The main contribution of this paper is a novel \ac{CL} framework that can achieve robust generalization to dynamic environments while retaining past knowledge. In the considered framework, a \ac{CL} agent deploys a capacity-limited memory to save previously observed environmental information. Then, a novel optimization problem is formulated to minimize the worst-case risk over all possible environments to ensure the robust generalization while balancing the memorization over the past environments.  However, it is generally not possible to know the change of dynamic environments, and deriving the worst-case risk is not feasible. To mitigate this intractability, the problem is relaxed with probabilistic generalization by considering risks as a random variable over environments. Then, data points are sampled from the memory to estimate the distribution of risks over environmental change so as to obtain predictors that are robust with unseen changes. We then provide theoretical analysis about the generalization and memorization of our framework with new insights that a tradeoff exists between them in terms of the memory size. Experiments show that our framework can achieve robust generalization performance for unseen target environments while retaining past experiences. The results show up to a 10\% gain in the generalization compared to memory-based \ac{CL} baselines.

\section{System Model}
\vspace{-0.0cm}
\label{sec:problem}

\subsection{Setup}
\label{subsec:setup}

Consider a single \ac{CL} agent equipped with \ac{AI} that performs certain tasks observing streams of data sampled from a dynamic environment. The agent is embedded with a \ac{ML} model parameterized by $\model \in \Theta \subset \mathbb{R}^d$, where $\Theta$ is a set of possible parameters and $d  > 0$. We assume that the agent continuously performs tasks under dynamic environments and updates its model. To perform a task, at each time $t$, the agent receives a batch of data $\batch = \{(X^{e_t}_i, Y^{e_t}_i)\}_{i=1}^{\batchsize}$ with input-output pairs $(X^{e_t}_i, Y^{e_t}_i)$ sampled/observed from a joint distribution $P(X^{e_t}, Y^{e_t})$ under the current environment $\env_t \sim \envwhole$, where $\envwhole$ represents the set of all environments. We assume that there exists a probability distribution $\mathcal{Q}$ over environments in $\envwhole$ \cite{eastwood2022probable}. $\mathcal{Q}$ can for example represent a distribution over changes to weather or cities for autonomous vehicles. It can also model a distribution over changes to rotation, brightness, or noise in image classification tasks. Hence, it is important for the agent to have robust generalization to such dynamic changes and to memorize past experiences. To this end, we use a memory $\mem$ of limited capacity $0 \leq |\mem| \leq \memsize$  so that the agent can save the observed data samples $\{(X^{e_t}_i, Y^{e_t}_i)\}_{i=1}^{\batchsize}$ in $\mem$ and can replay them when updating its model $\model_t$.

To measure the performance of $\model_t$ under the current environment $\env_t$, we  consider the statistical risk $\srisk^{\env_t}(\model_t)$ $= \E_{P(X^{e_t}, Y^{e_t})}$$[l(\theta_t(X^{e_t}),$ $Y^{e_t} )]$, where $\loss(\cdot)$ is a loss function, e.g., cross-entropy loss. Since we usually do not know the distribution $P(X^{e_t}, Y^{e_t})$, we also consider the empirical risk $\erisk^{\env_t}(\model_t) = \frac{1}{\batchsize} \sum_{i=1}^{\batchsize} \loss(\model_t(X^{e_t}), Y^{e_t}))$. 

\subsection{Problem Formulation}
\label{subsec:problem}
As the agent continuously experiences new environments, it is important to maintain the knowledge of previous environments (\emph{memorization}) and generalize robustly to any unseen environments (\emph{generalization}). This objective can be formulated into the following optimization problem:
\begin{align}
	\min_{\model_t \in \Theta} \ \ \ \ \sum_{\tau \in \mem} \frac{1}{|\mem|} \erisk^{\env_\tau} (\model_t) + \max_{\env_t \sim \envwhole} \balance \srisk^{\env_t}(\model_t) , \label{first_prob}
\end{align}
where $|\mem|$ is the size of the current memory and $\balance > 0 $ is a coefficient that balances between the terms in \eqref{first_prob}. The first term is the \emph{memorization performance} of the current model $\model_t$ with respect to the past experienced environments $\env_\tau, \forall \tau \in [1, \dots, |\mem|],$ at time $t$. The second term corresponds to the \emph{worst-case performance} of $\model_t$ to any possible environment $\env_t \in \envwhole$. Since the change of environments is dynamic and not predictable, we consider all possible cases to measure the robustness of the current model. This problem is challenging because the change of $\env_t$ is unpredictable and its information, e.g., label, is not available. Meanwhile, we only have limited access to the data from the observation in $\mem$. 

Since each environment follows a probability distribution $\mathcal{Q}$, $\srisk^{\env_t}$ can be considered as a random variable. Then, we rewrite \eqref{first_prob} as follows
\begin{align}
	&\min_{\model_t \in \Theta, \gamma \in \mathcal{R}} \ \ \ \ \sum_{\tau \in \mem} \frac{1}{|\mem|} \erisk^{\env_\tau} (\model_t) + \gamma \balance, \label{sec_prob}   \\
	&\quad \ \text{s.t.} \quad \quad \quad \Prob[\srisk^{\env_t} (\model_t) \leq \gamma ] = 1, 
	\label{sec_const}
\end{align}
where the probability in \eqref{sec_const} considers the randomness in $\env_t \sim \mathcal{Q}$. However, the problem is still challenging because constraint \eqref{sec_const} must be always satisfied. This can be too restrictive in practice due to the inherent randomness in training (e.g., environmental changes). To make the problem more tractable, we use the framework of \ac{PDG}\cite{eastwood2022probable}. In \ac{PDG} , we relax constraint \eqref{sec_const} with probability $\alpha \in (0, 1)$ as below
\begin{align}
	&\min_{\model_t \in \Theta, \gamma \in \mathcal{R}} \ \ \ \ \sum_{\tau \in \mem} \frac{1}{|\mem|} \erisk^{\env_\tau} (\model_t) + \gamma \balance, \label{third_prob}   \\
	&\quad \ \text{s.t.} \quad \quad \quad \Prob[\srisk^{\env_t} (\model_t) \leq \gamma ] \geq \alpha.
	\label{third_const}
\end{align}
Hence, constraint \eqref{third_const} now requires the risk of $\model_t$ to be lower than $\gamma$ with probability at least $\alpha \in (0, 1)$. However, $\mathcal{Q}$ is generally unknown, so the probability term in \eqref{third_const} is still intractable. Since risk $\srisk^\env (\cdot)$ is a random variable, we can consider a certain probability distribution $f_\srisk$ of risks over environment $\env \sim \envwhole$ \cite{eastwood2022probable}. Here, $f_\srisk$ can capture the sensitivity of $\model$ to different environments. Then, we can rewrite the problem by using the \ac{CDF} $\scdf$ of $f_\srisk$ as:
\begin{align}
	&\min_{\model_t \in \Theta} \ \ \ \ \sum_{\tau \in \mem} \frac{1}{|\mem|} \erisk^{\env_\tau} (\model_t) + \scdf^{-1}(\alpha; \model_t) \balance, \label{fourth_prob}   
\end{align}
where $\scdf^{-1}(\alpha; \model_t) = \inf\{ \gamma: \Prob [\srisk^{\env_t} (\model_t) \leq \gamma] \geq \alpha \}$. Now, to estimate $\scdf$ through its empirical version $\ecdf$, we sample a batch of data $\rbatch$ from $\mem$ and another batch $\batch$ from $\env_t$. Since $\scdf$ is an unknown distribution, we can use kernel density estimation or Gaussian estimation \cite{eastwood2022probable} for $\scdf$ using the sampled data. This approach is similar to minimizing empirical risks instead of statistical risks in conventional training settings \cite{shalev2009stochastic}. For the computational efficiency, we also sample a batch $\mbatch$ from $\mem$ to approximate the performance of $\model_t$ over the memory. We then obtain the following problem
\begin{align}
&\min_{\model_t \in \Theta} \ \ \ \ \sum_{\tau \in \mbatch} \frac{1}{|\mbatch|} \erisk^{\env_\tau} (\model_t) + \ecdf^{-1}(\alpha; \model_t) \balance. \label{final_prob}   
\end{align}

The above problem \eqref{final_prob} only uses data sampled from $\mem$ and $\env_t$, thereby mitigating the intractability due to the unknown distribution in \eqref{fourth_prob}. Hence, \eqref{final_prob} can be solved using gradient-based methods after the distribution estimation. We summarize our method in Algorithm \ref{algorithm1} with Gaussian estimation. Since we use data samples in $\mem$ to estimate problem \eqref{fourth_prob}, the size of $\mem$ naturally represents the richness of estimation $\ecdf$. As we have a larger memory size, we can save more environmental information and can achieve more robust generalization. However, a large $|\mem|$ also means that an agent has more information to memorize. Finding a well-performing model across all environments in $\mem$ is generally a challenging problem \cite{knoblauch2020optimal}. To capture this tradeoff between the generalization and memorization, next, we analyze the impact of the memory size $|\mem|$ on the memorization and generalization performance.

\begin{algorithm} [t!] 
	\caption{Proposed Algorithm } \label{algorithm1}
	\KwInput{Model $\model$, probability of generalization $\alpha$, learning rate $\eta$, memory $\mem$, \ac{CDF} of normal distribution $\Phi(\cdot)$, batches $\batch, \rbatch, \mbatch$, and balance coefficient $\balance$.} 
	\For{$t = 0$ to $T-1$}{
		%Calculate $\mu = \frac{1}{\rbatch}\sum_{\tau =1}^{ \rbatch}\loss(\model_t(X^{\env_\tau}), Y^{\env_\tau})$ and  $\sigma^2 = \frac{1}{\rbatch-1} \sum_{\tau =1}^{ \rbatch}(\loss(\model_t(X^{\env_\tau}), Y^{\env_\tau}) - \mu)^2$ from sampled datasets in $\rbatch$
		Sample a batch of data $\batch$ from $\env_t$;\\
		Calculate the mean $\mu_t$ and vairance $\sigma^2_t$ of risks $\erisk(\model_t)$ using sampled datasets in $\batch$ and $\rbatch$ ;\\
		Compute $\alpha$ quantile of the estimated Gaussian distribution $L_G \leftarrow \mu_t + \sigma_t^2 \Phi^{-1}(\alpha)$;\\
		Calculate $L_M\leftarrow \sum_{\tau \in \mbatch} \frac{1}{|\mbatch|} \erisk^{\env_\tau} (\model_t)$ using sampled data in $\mbatch$;\\
		Update $\theta_t \leftarrow \theta_t - \eta \nabla_{\model_t} (\rho L_G + L_M)$
	}
\end{algorithm}

\section{Tradeoff between Memorization and Generalization}
\label{sec:theory}
We now study the impact of the memory size $|\mem|$ on the memorization.  
Motivated by \cite{peng2023ideal}, we assume that a global solution $\model_t^*$ exists for every environment $\tau \in [1, \dots, |\mem|]$ at time $t$ such that $\model_t^* = \arg\min_{\model \in \Theta} \sum_{\tau=1}^{|\mem|} \srisk^{\env_\tau}(\theta).$ If such $\model_t^*$ does not exist, memorization would not be feasible. Then, we can have the following theorem.
\begin{theorem} \label{thm1}
For time $t$, let $\modelmem$ be a global solution for all environments $\tau \in [1, \dots, |\mem|]$ in $\mem$ and suppose loss function $\loss (\cdot)$ to be $\convex$-strongly convex and $\lip$-Lipschitz-continuous. Then, for the current model $\model_t \in \Theta$ and $\epsilon > 0$, we have
\begin{align}
	&\Prob 
	\Bigg[
	\bigcap_{\tau = 1}^{|\mem|} 
	\left\{
	\srisk^{\env_\tau} (\model_t) - \srisk^{\env_\tau}(\modelmem) \leq \epsilon 
	\right\}
	\Bigg] \ka
	& \quad \geq 1 - \frac{4|\mem| \lip^2}{\convex|\mbatch|
		\Big(
		\epsilon - \sqrt{
			2\lip^3 ||\model_t - \modelem ||/\convex
		}
		\Big)
	},	
\end{align}
where $\modelem$ is the empirical solution for all environments $\tau \in [1, \dots, |\mem|]$.
\end{theorem}
\begin{proof}
	We first state one standard lemma used in the proof as below
	\begin{lemma} \label{lem1}
		(From \cite[ Theorem 5]{shalev2009stochastic}) For $\model \in \Theta$, a certain environment $\tau$, and its optimal solution $\model^*$, with probability at least $1-\delta$, we have
		\begin{align}
			\Prob \hspace{-0.5mm}
			\underbrace{ \bigg[ \hspace{-0.5mm}
			\srisk^{\env_\tau} \hspace{-0.5mm} (\hspace{-0.1mm}\model\hspace{-0.1mm}) \hspace{-0.7mm} - \hspace{-0.7mm}\srisk^{\env_\tau} \hspace{-0.5mm}(\hspace{-0.1mm} \model^*\hspace{-0.1mm}) \hspace{-0.3mm} \leq \hspace{-0mm} \sqrt{ \hspace{-0.5mm} \frac{2\lip^2}{\convex} (\hspace{-0.0mm} \erisk^{\env_\tau} \hspace{-0.5mm} (\hspace{-0.3mm}\model\hspace{-0.3mm})\hspace{-0.7mm} - \hspace{-0.7mm} \erisk^{\env_\tau} \hspace{-0.5mm} (\hspace{-0.1mm} \hat{\model} \hspace{-0.1mm})  \hspace{-0.0mm} ) }    \hspace{-0.7mm} + \hspace{-0.7mm} \frac{4\lip^2}{\delta \convex |\mbatch|} \hspace{-0mm}
			\bigg]}_{A} \hspace{-0.5mm}
			\geq \hspace{-0.5mm} 1 \hspace{-0.7mm} - \hspace{-0.7mm} \delta.
		\end{align}
	\end{lemma}
From the $\lip$-Lipschtiz assumption, we have the following inequality
		\begin{align}
	\Prob \hspace{-0.7mm}
	\left[ \hspace{-0.7mm}
	\srisk^{\env_\tau} \hspace{-0.5mm} (\model) \hspace{-0.7mm} - \hspace{-0.7mm}\srisk^{\env_\tau} \hspace{-0.5mm}(\model^*) \hspace{-0.7mm} \leq \hspace{-1.0mm} \sqrt{\frac{2\lip^2}{\convex}} \hspace{-0.7mm} \sqrt{ \lip ||\model - \hat{\model} ||}  \hspace{-0.7mm} + \hspace{-0.7mm} \frac{4\lip^2}{\delta \convex |\mbatch|} \hspace{-0.7mm}
	\right] \hspace{-0.9mm}
	\geq A .
	%\geq \hspace{-0.7mm} 1 \hspace{-0.7mm} - \hspace{-0.7mm} \delta.
\end{align}
For time $t$, since $\modelmem$ is a global solution for all $\tau \in [1, \dots, |\mem|]$, the following holds
\begin{align}
		&\Prob \hspace{-0.7mm} 
	\bigg[ \hspace{-0.5mm} \bigcap_{\tau = 1}^{|\mem|} \hspace{-1.0mm}
	\bigg\{    \hspace{-0.7mm}
	\srisk^{\env_\tau} \hspace{-0.5mm} (\hspace{-0.3mm} \model_t\hspace{-0.3mm} ) \hspace{-0.7mm} - \hspace{-0.7mm}\srisk^{\env_\tau} \hspace{-0.5mm}(\hspace{-0.3mm} \modelmem\hspace{-0.3mm} ) \hspace{-0.7mm} \leq \hspace{-1.0mm} \sqrt{\frac{2 \hspace{-0.5mm} \lip^2}{\convex}} \hspace{-0.7mm} \sqrt{ \hspace{-0.7mm} \lip ||\hspace{-0.3mm} \model_t \hspace{-0.7mm} -  \hspace{-0.7mm} \modelem \hspace{-0.3mm} ||}  \hspace{-0.7mm} + \hspace{-0.7mm} \frac{4\lip^2}{\delta \convex |\mbatch|} \hspace{-0.7mm}
	\bigg\} \hspace{-0.7mm}
	\bigg] \label{righthand} \\
	&\geq
	\sum_{\tau=1}^{|\mem|}\bigg[
	\srisk^{\env_\tau} \hspace{-0.5mm} (\model_t) \hspace{-0.7mm} - \hspace{-0.7mm}\srisk^{\env_\tau} \hspace{-0.5mm}(\modelmem) \hspace{-0.7mm} \leq \hspace{-1.0mm} \sqrt{\frac{2\lip^2}{\convex}} \hspace{-0.7mm} \sqrt{ \lip ||\model_t - \modelem ||}  \ka
	 &\quad + \frac{4\lip^2}{\delta \convex |\mbatch|} \hspace{-0.7mm}
	\bigg] - (|\mem| -1) \ka
	&\geq |\mem| (1-\delta) - (|\mem| - 1) = 1 - |\mem|\delta, \label{plug}
\end{align}
where the first inequality results from $\Prob [\bigcap_{i=1}^{n} A_i] \geq \sum_{i=1}^{n} \Prob[A_i]$ $- (n-1)$ and the second inequality is from the Lemma \ref{lem1}. We set the right-hand side of \eqref{righthand} to $\epsilon$ and the expression of $\delta$ as 
\begin{align}
    \delta \hspace{-0.0mm} = \hspace{-0.0mm} \frac{4\lip^2}{\convex|\mbatch| 
	\big(
	\epsilon - \sqrt{
	2\lip^3 ||\model_t - \modelem ||/\convex
    }
	\big)
    }.	
\end{align}
By plugging the derived $\delta$ into \eqref{plug}, we complete the proof.
\end{proof}
From Theorem \ref{thm1}, we observe that as the memory size $|\mem|$ increases, the probability that the difference between risks of $\model_t$ and $\modelmem$ is smaller than $\epsilon$ decreases. As we have more past experience in $\mem$, it becomes more difficult to achieve a global optimum for all experienced environments. For instance, if we have only one environment in $\mem$, finding $\model_{\mem}^*$ will be trivial. However, as we have multiple environments in $\mem$, $\model_t$ will more likely deviate from $\model^*_{\mem}$.

Next, we present the impact of the memory size $|\mem|$ on the generalization performance. 
\begin{prop} \label{prop1}
	Let $\mathcal{\hat{F}}_t$ denote the space of possible estimated risk distributions over $|\mem|$ environments and $\mathcal{N}_\epsilon(\mathcal{\hat{F}}_t)$ be the $\epsilon$ covering number of $\mathcal{\hat{F}}_t$ for $\epsilon > 0$ at time $t$. Then, for $\alpha \in (0, 1)$, we have the following 
	\begin{align}
		&\Prob 
		\left[ \sup_{\model_t \in \Theta}
		\scdf^{-1}(\alpha - \text{Bias}(\model_t, \erisk)) - \ecdf^{-1}(\alpha) > \epsilon 
		\right] \ka
		&\quad \leq \mathcal{O}(\mathcal{N}_{\epsilon/16}(\mathcal{\hat{F}}_t) \exp(-\frac{|\mem| \epsilon^2}{16}) )  ,
	\end{align}
where $ \text{Bias}(\model_t, \hspace{-0.7mm} \erisk)) \hspace{-0.7mm}= \hspace{-0.7mm} \sup_{\model_t \in \Theta, \gamma \in \mathcal{R}} \hspace{-0.7mm} \scdf \hspace{-0.5mm} (\hspace{-0.3mm} \model_t \hspace{-0.3mm}) - \E_{e_1, \dots, e_{|\hspace{-0.3mm} \mem \hspace{-0.3mm} |}} \ecdf \hspace{-0.5mm} (\hspace{-0.3mm} \model_t\hspace{-0.3mm} ) $.
\end{prop}
\begin{proof}
	The complete proof is omitted due to the space limitation. The proposition can be proven by leveraging \cite[Theorem 1]{eastwood2022probable} and using the memory $\mem$ instead of the already given data samples.
\end{proof}
We can see that both the bias term and the upper bound are a decreasing function of the memory size $|\mem|$. As we have more information about the environments in $\mem$, the model $\model_t$ becomes more robust to the dynamic changes of the environments, thereby improving generalization.

From Theorem \ref{thm1} and Proposition \ref{prop1}, we can see that a tradeoff exists between memorization and generalization in terms of $|\mem|$. As $|\mem|$ increases, we have more knowledge about the entire environment $\envwhole$, and $\model_t$ can be more robust to changes in $\env_t$. However, maintaining the knowledge of all the stored environments in $\mem$ becomes more difficult. Hence, the deviation of $\model_t$ from $\model^*_{\mem}$ will increase.  Essentially, this observation is similar to the performance of FedAvg algorithm \cite{mcmahan2017communication} in \ac{FL}. As the number of participating clients increases, the global model achieves better generalization. However, it does not perform very well on each dataset of clients. The global model usually moves toward just the average of each client's optima instead of the true optimum \cite{karimireddy2020scaffold}.
\section{Experiments}
\label{sec:experiments}
We now conduct experiments to evaluate our proposed algorithm and to validate the analysis. We use the rotated MNIST datasets \cite{lopez2017gradient}, where digits are rotated with a fixed degree. Here, the rotation degrees represent an environmental change inducing distributional shifts to the inputs. The training datasets are rotated by $0^\circ$ to $150^\circ$ by interval of $25^\circ$. Hence, the agent classifies all MNIST digits for seven different rotations, i.e., environments. To measure the generalization performance, we left out one target degree from the training datasets. Unless specified otherwise, we use a two-hidden fully-connected layers with 100 ReLU units, a \ac{SGD} optimizer and a learning rate of 0.1. For the memory, we adopted a replay buffer in \cite{rolnick2019experience} with reservoir sampling with $\memsize = 10000$. We also set $|\batch |= 512$ and $|\mbatch| = 64$. To estimate $\scdf$, we used Gaussian estimation by sampling three batches of size $|\rbatch| =64$ from the memory with $\alpha = 0.9999$ and $\balance = 0.5.$ We average our results over ten random seeds

\begin{table*}[t!] 
\begin{tabular}{ccccccccc}
	\hline
	\multirow{2}{*}{\textbf{Method}} &
	\multicolumn{2}{c}{\textbf{0$^\circ$}} &
	\multicolumn{2}{c}{50$^\circ$} &
	\multicolumn{2}{c}{100$^\circ$} &
	\multicolumn{2}{c}{150$^\circ$} \\
	&
	Avg &
	Target &
	Avg &
	Target &
	Avg &
	Target &
	Avg &
	Target \\ \hline
	\textbf{Ours} &
	\textbf{86.79$\pm$0.7} &
	\textbf{60.54$\pm$0.7} &
	\textbf{89.07$\pm$0.2} &
	\textbf{81.27$\pm$0.5} &
	\textbf{89.53$\pm$0.4} &
	\textbf{81.46$\pm$0.7} &
	\textbf{87.08$\pm$0.2} &
	\textbf{60.12$\pm$1.3} \\
	ER\cite{rolnick2019experience} &
	85.68$\pm0.5$ &
	58.43$\pm0.5$ &
	88.72$\pm0.2$ &
	80.92$\pm0.6$ &
	88.63$\pm0.3$ &
	80.58$\pm0.7$ &
	84.82$\pm0.4$ &
	54.11$\pm1.0$ \\
	DER\cite{buzzega2020dark} &
	81.58$\pm0.4$ &
	52.24$\pm0.7$ &
	84.02$\pm0.5$ &
	75.62$\pm0.5$ &
	83.74$\pm0.6$ &
	75.36$\pm1.1$ &
	80.79$\pm0.7$ &
	49.6$\pm1.3$ \\
	DER++\cite{buzzega2020dark} &
	83.71$\pm0.8$ &
	55.50$\pm1.1$ &
	86.34$\pm0.8$ &
	78.04$\pm0.9$ &
	86.38$\pm0.6$ &
	78.43$\pm0.5$ &
	83.24$\pm0.8$ &
	52.70$\pm1.0$ \\
	EWC\_ON\cite{schwarz2018progress} &
	83.10$\pm0.4$ &
	54.42$\pm0.8$ &
	85.63$\pm0.4$ &
	77.23$\pm0.1$ &
	85.52$\pm0.6$ &
	77.01$\pm1.0$ &
	82.66$\pm0.2$ &
	51.38$\pm1.4$ \\
	GSS\cite{aljundi2019gradient} &
	83.23$\pm0.4$ &
	54.72$\pm0.8$ &
	85.35$\pm0.6$ &
	77.10$\pm0.5$ &
	85.55$\pm0.5$ &
	76.97$\pm1.0$ &
	82.60$\pm0.3$ &
	51.05$\pm1.2$ \\
	HAL\cite{chaudhry2021using} &
	83.15$\pm0.3$ &
	54.52$\pm1.0$ &
	85.37$\pm0.6$ &
	77.45$\pm0.6$ &
	85.75$\pm0.3$ &
	77.25$\pm0.6$ &
	82.69$\pm0.3$ &
	52.52$\pm0.4$ \\ \hline
\end{tabular}
\vspace{-0.0cm}
\caption{Accuracy on the held-out test datasets and average accuracy across all rotations.}
\label{tab1}
\end{table*}

In Table \ref{tab1}, we compare our algorithm against five memory-based \ac{CL} methods (ER \cite{rolnick2019experience}, DER\cite{buzzega2020dark}, DER++\cite{buzzega2020dark}, , GSS \cite{aljundi2019gradient}, and HAL \cite{chaudhry2021using}) and one regularization-based \ac{CL} method (EWC-ON \cite{schwarz2018progress}). We present results only on four rotations due to space limitations. In Table \ref{tab1}, \lq Avg\rq means average accuracy on test datasets across all rotations. \lq Target\rq  measures the generalization performance on unseen datasets. We can see that our algorithm outperforms the baselines in terms of both the generalization performance and average accuracy across all rotations. Hence, our algorithm achieves better generalization while not forgetting the knowledge of past environments. We can also observe that our algorithm achieves robust generalization to challenging rotations ($0^\circ$ and $150^\circ$).
\begin{figure*}[t!]
	\centering	
	\begin{subfigure}[t]{0.42\textwidth}
		\centering	
		\includegraphics[width=\textwidth]{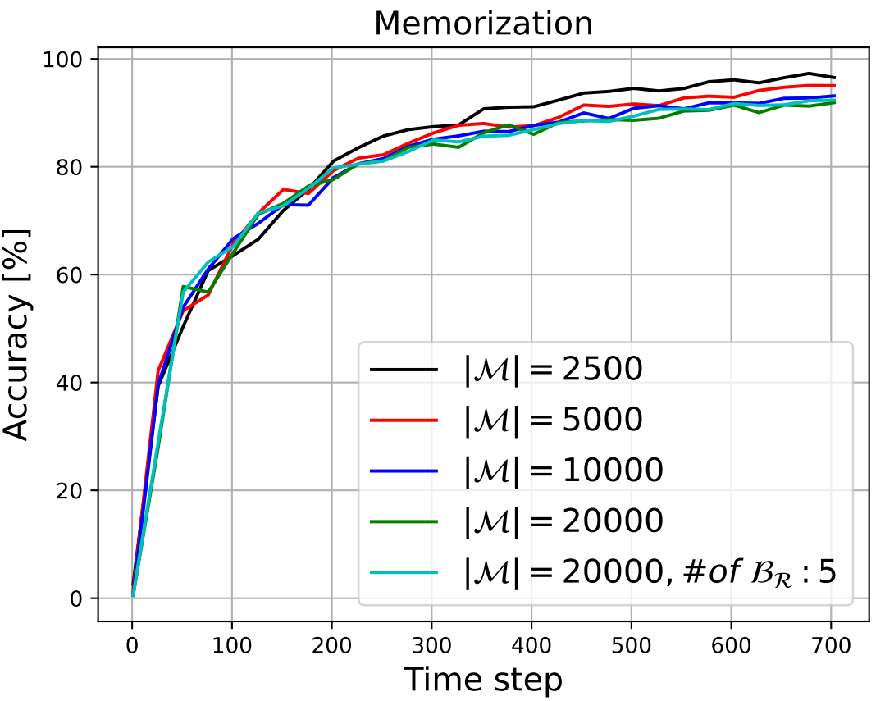}
		%\caption{Performance of \ac{NBS} with different $K$}
		\label{fig:mem}
	\end{subfigure} \hfill
	\begin{subfigure}[t]{0.42\textwidth}
		\centering
		\includegraphics[width=\textwidth]{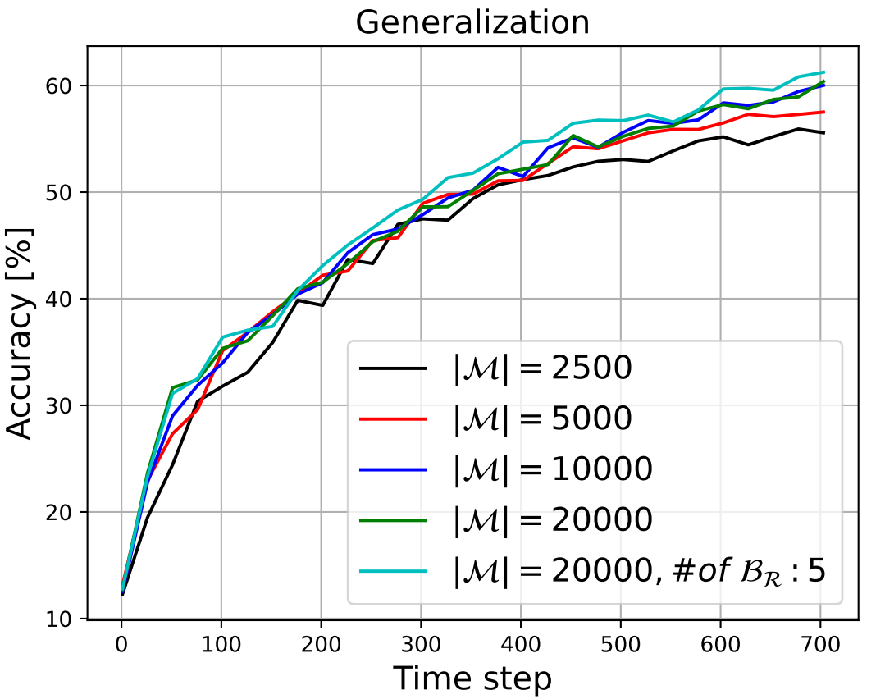}
		%\caption{Performance of \ac{SUM} with different $K$}
		\label{fig:gen}
	\end{subfigure}\hfill
	\vspace{-0.3cm}
	\caption{Impact of the memory size on the memorization and generalization.}
	\vspace{-0.0cm}
	\label{fig:impact}
	\vspace{-0.3cm}
\end{figure*}

\begin{table}[t!]
	\centering
\begin{tabular}{ccc}
	\hline
	\multirow{2}{*}{\textbf{$\alpha$}} & \multicolumn{2}{c}{\textbf{$150^\circ$}} \\
	& Avg             & Target         \\ \hline
	\multicolumn{1}{l}{0.99999}        & 86.44$\pm0.6$   & 59.20$\pm1.3$  \\
	\textbf{0.9999}                             & \textbf{87.09$\pm$0.2}   & \textbf{60.12$\pm$1.3}  \\
	0.99                               & 86.61$\pm0.4$   & 58.78$\pm1.2$  \\
	0.9                                & 85.61$\pm0.3$   & 57.76$\pm0.7$  \\
	0.5                                & 83.54$\pm0.6$   & 54.17$\pm0.8$  \\
	0.3                                & 81.55$\pm0.9$   & 51.52$\pm1.0$  \\ \hline
\end{tabular}
\caption{Impact of $\alpha$ on the performance.}
\label{tab2}
\vspace{-0.3cm}
\end{table}
In Table \ref{tab2}, we show the impact of $\alpha$ on the performance of our algorithm. We left out the $150^\circ$ datasets to measure the generalization. We can observe that as $\alpha$ decreases, the trained model does not generalize well. This is because $\alpha$ represents a probabilistic guarantee of the robustness to environments as shown in \eqref{third_prob}. However, too large $\alpha$ can lead to overly-conservative models that cannot adapt to dynamic changes.

In Fig. \ref{fig:impact}, we present the impact of the memory size $\memsize$ on the memorization and generalization. We left out the $150^\circ$ datasets to measure the generalization. For the memorization, we measured the accuracy of $\model_t$ on whole data samples in $\mem$ at time step $t$ as done in \cite{cai2021online}. As the memory size $\memsize$ increases, a trained model generalizes better while struggling with memorization. This observation corroborates our analysis in Sec. \ref{sec:theory} as well as empirical findings in \cite{cai2021online}. However, for $\memsize =20000$, we can see that the generalization performance does not improve. This is because we solved problem \eqref{fourth_prob} through approximation by sampling batches from the memory $\mem$. Hence, if we do not sample more batches $\rbatch$ with large $\memsize$, we cannot fully leverage various environmental information in $\mem$. We can observe that the generalization improves by sampling more $\rbatch$ batches for the distribution estimation.

\section{Conclusion}
\label{sec:con}

In this paper, we have developed a novel \ac{CL} framework that provides robust generalization to dynamic environments while maintaining past experiences. We have utilized a memory to memorize past knowledge and achieve domain generalization with a high probability guarantee. We have also presented new theoretical insights into the impact of the memory size on the memorization and generalization performance. The experimental results show that our framework can achieve robust generalization to unseen target environments during training while retaining past experiences.

\begin{comment}
\begin{figure}[htb]

\begin{minipage}[b]{1.0\linewidth}
  \centering
  \centerline{\includegraphics[width=8.5cm]{image1}}
%  \vspace{2.0cm}
  \centerline{(a) Result 1}\medskip
\end{minipage}
%
\begin{minipage}[b]{.48\linewidth}
  \centering
  \centerline{\includegraphics[width=4.0cm]{image3}}
%  \vspace{1.5cm}
  \centerline{(b) Results 3}\medskip
\end{minipage}
\hfill
\begin{minipage}[b]{0.48\linewidth}
  \centering
  \centerline{\includegraphics[width=4.0cm]{image4}}
%  \vspace{1.5cm}
  \centerline{(c) Result 4}\medskip
\end{minipage}
%
\caption{Example of placing a figure with experimental results.}
\label{fig:res}
%
\end{figure}
\end{comment}

\newpage{}
\bibliographystyle{IEEEbib}
\bibliography{Bibtex/StringDefinitions,Bibtex/IEEEabrv,Bibtex/mybib}

\end{document}

%% file: SupportDocuments/acronym.tex
\acrodef{IoT}{Internet of Things}
\acrodef{BS}{base station}
\acrodef{pdf}{probability density function}
\acrodef{i.i.d.}{independent and identically distributed}
\acrodef{CDF}{cumulative distribution function}
\acrodef{FL}{federated learning}
\acrodef{ML}{machine learning}
\acrodef{AI}{artificial intelligent}
\acrodef{SGD}{stochastic gradient descent}
\acrodef{SG}{stochastic gradient}
\acrodef{MAC}{multiply-accumulate}
\acrodef{CNN}{convolutional neural network}
\acrodef{DNN}{deep neural network}
\acrodef{QNN}{quantized neural network}
\acrodef{OFDMA}{orthogonal frequency domain multiple access}
\acrodef{NAS}{neural architecture search}
\acrodef{PDG}{probable domain generalization}
\acrodef{CL}{continual learning}

%% file: SupportDocuments/defmetric.tex
\usepackage{color}
\usepackage{dsfont}
\usepackage{bbm}

\newtheorem{theorem}{Theorem}
\newtheorem{lemma}{Lemma}
\newtheorem{prop}{Proposition}   

\newtheorem{assumption}{Assumption}

\newcommand{\E}{\mathbb{E}}
\newcommand{\Prob}{\mathbb{P}}

\newcommand{\srisk}{\mathcal{R}}
\newcommand{\erisk}{\hat{\mathcal{R}}}

\newcommand{\ka}{\nonumber \\}
\newcommand{\mem}{\mathcal{M}_t}
\newcommand{\memsize}{|\mathcal{M}|}
\newcommand{\model}{\theta}
\newcommand{\env}{e}
\newcommand{\envwhole}{\mathcal{E}}
\newcommand{\loss}{l}
\newcommand{\scdf}{F_{\srisk}}
\newcommand{\ecdf}{F_{\erisk}}
\newcommand{\convex}{\lambda}
\newcommand{\lip}{L}
\newcommand{\balance}{\rho}
\newcommand{\modelmem}{\model_{\mem}^*}
\newcommand{\modelem}{\hat{\model}_{\mem}}
\newcommand{\batch}{\mathcal{B}}
\newcommand{\rbatch}{\mathcal{B}_{\srisk}}
\newcommand{\mbatch}{\mathcal{B}_{M}}
\newcommand{\batchsize}{|\batch|}